\documentclass[preprint,review,12pt]{elsarticle}

\usepackage{amssymb}
\usepackage{amsmath}
\usepackage{graphicx}
\usepackage[retainorgcmds]{IEEEtrantools}
\usepackage[colorlinks]{hyperref}
\hypersetup{
  citecolor = {blue}
}
\usepackage{siunitx}
\usepackage{pdflscape}

\usepackage{enumitem}
\usepackage[normalem]{ulem}
\useunder{\uline}{\ul}{}
\usepackage{setspace} 
\usepackage[utf8]{inputenc}
\usepackage[linesnumbered, ruled, vlined]{algorithm2e}
\usepackage{rotating}
\usepackage{nicefrac}
\usepackage{xspace}
\usepackage{float}
\usepackage{caption}
\usepackage{subcaption}
\usepackage{rotating, lscape, longtable, tabu, booktabs, siunitx, multirow}
\usepackage[capitalise, noabbrev]{cleveref}
\usepackage{verbatim}
\usepackage{mathtools}

\usepackage{tkz-graph}

\usepackage{algorithmic}
\usepackage{xcolor}
\usepackage{booktabs}
\usepackage{threeparttable}
\usepackage{tabularx} 
\usepackage{caption}

\usetikzlibrary{arrows,arrows.meta,shapes,decorations.pathmorphing, decorations.markings, positioning}
\let\svtikzpicture\tikzpicture
\def\tikzpicture{\noindent\svtikzpicture}

\newcommand{\uset}[1]{\ifmmode\left\{\,#1\,\right\}\else\{\,#1\,\}\fi}
\newcommand{\ulst}[1]{\ifmmode\left[\,#1\,\right]\else[\,#1\,]\fi}
\newcommand{\upar}[1]{\ifmmode\left(\,#1\,\right)\else(\,#1\,)\fi}
\newcommand{\uioc}[1]{\ifmmode\left(\,#1\,\right]\else(\,#1\,]\fi}
\newcommand{\uico}[1]{\ifmmode\left[\,#1\,\right)\else[\,#1\,)\fi}

\newtheorem{theorem}{Theorem}

\newenvironment{proof}{\noindent\textbf{Proof}.}{\hfill$\square$}
\newtheorem{definition}{Definition}

\journal{Journal name}

\begin{document}

\begin{frontmatter}

\title{Deep Learning Meets Oversampling: A Learning Framework to Handle Imbalanced Classification}

\author[unifal]{Sukumar Kishanthan\corref{cor1}}
\ead{kishanthansukumar@gmail.com}

\author[label2]{Asela Hevapathige}
\ead{asela.hevapathige@anu.edu.au}

\address[unifal]{Faculty of Engineering, University of Ruhuna, Galle 80000, Sri Lanka}

\cortext[cor1]{Corresponding author}

\address[label2]{College of Engineering, Computing and Cybernetics, The Australian National University, Canberra ACT 2601, Australia}

\begin{abstract}
Despite extensive research spanning several decades, class imbalance is still considered to be a profound difficulty for both machine learning and deep learning models. While data oversampling is the foremost technique to address this issue, traditional sampling techniques are often decoupled with the training phase of the predictive model, resulting suboptimal representations. To address this, we propose a novel learning framework that is capable in generating synthetic data instances in a data-driven manner. The proposed framework formulates the oversampling process as a composition of discrete decision criteria, thereby enhancing the representation power of the model's learning process. Extensive experiments on the imbalanced classification task demonstrates the superiority of our framework over the state-of-the-art algorithms.
\end{abstract}


\end{frontmatter}


\section{Introduction} \label{sec:intro}

Class imbalance is a non-trivial and enduring challenge in machine learning and data mining, where the distribution of classes within a dataset is skewed. This disproportion often leads to biased model training, making the classifier inclined towards predicting the majority class in the inference phase\cite{guo2008class,johnson2019survey}. The class imbalance problem cannot be readily overlooked, as many real-world datasets related to critical tasks, such as those used in the medical field for disease identification, the finance sector for fraud detection, and network intrusion datasets used in cyber security, exhibit such asymmetric class distributions \cite{cieslak2006combating,japkowicz2002class,al2021financial}.
\\
Existing machine learning and deep learning approaches primarily utilize resampling techniques to tackle class imbalance which involves adjustment techniques to balance the class distribution in datasets \cite{khushi2021comparative,marques2013suitability}. Among diverse resampling techniques, Oversampling approaches are commonly preferred for addressing class imbalance mainly due to their inherent ability to equalize the class distribution while preserving data semantics and achieving superior performance. There has been a plethora of different oversampling techniques proposed in the literature, ranging from traditional approaches \cite{chawla2002smote,tang2008svms,han2005borderline,he2008adasyn,douzas2018improving} to those based on deep learning \cite{ando2017deep,dablain2022deepsmote,karunasingha2023oc}. Traditional oversampling algorithms are often applied as a pre-processing step and are decoupled with the classifier training process. This introduces a significant limitation as synthetic data generated during oversampling may not fully align with the semantics of the downstream classification task. On the other hand, existing deep learning-based oversampling approaches have expansive parameter search space and higher training complexities, often leading to model overfitting and subsequent poor generalization of test data.

In this work,  we propose a novel deep learning-based oversampling framework, namely \emph{AutoSMOTE} ( \textbf{\underline{Auto}}mated \textbf{\underline{S}}ynthetic \textbf{\underline{M}}inority \textbf{\underline{O}}versampling \textbf{\underline{TE}}chnique), to address the aforementioned limitations. AutoSMOTE is tailored to generate synthetic minority samples in a data-driven manner. Our approach is an end-to-end architecture jointly optimized alongside the classifier. AutoSMOTE incorporates a set of learnable discrete decision criteria to define the oversampling process which helps to reduce the parameter search space and training complexities of the model by a significant margin, thereby mitigating model overfitting and achieving better generalization. Our research contributions are summarized as follows:

\begin{itemize}
  \item \textbf{Novel Perspective on Oversampling: } We formulate the oversampling process as a composition of discrete decision criteria, that enriches the capability of the model to represent nuanced semantics among synthetic data instances. 
  \item \textbf{Deep Learning Framework: } We propose a novel deep oversampling framework, AutoSMOTE, to handle imbalanced classification. To this end, two variants are proposed. Further, we theoretically analyze the generalization error bounds of these variants to provide valuable insights on their performance.

  \item \textbf{Empirical Performance: } We extensively evaluate the proposed framework on imbalanced classification tasks with a variety of datasets, demonstrating its superior performance.
\end{itemize}

We organize the remainder of this paper as follows:
Section~\ref{sec:related} provides the related literature of our work.
Section~\ref{sec:methodology} formally defines our problem and present the proposed methodology. We provide theoretical insights on our work in Section \ref{sec:theory}. The experimental design is explained in Section \ref{sec:expriment}, and the empirical performance of our model is presented in Section \ref{sec:result}. Finally, conclusions of our work alongside potential future extensions are discussed in Section~\ref{sec:conclusions}.

\section{Related Work} \label{sec:related}

Machine learning techniques for alleviating class imbalance issue can be categorized under three primary directions, namely, Data level, Algorithmic level, and Hybrid level approaches \cite{sharma2022review,sowah2021hcbst}. Data-level approaches mainly focused on modifying training set data using resampling techniques such as minority class oversampling \cite{gosain2017handling} and majority class undersampling \cite{devi2020review}, with the goal of balancing the class distribution. On the other hand, algorithmic-level approaches modify the learning algorithm using techniques such as cost-sensitive learning \cite{ling2008cost} to give more importance to minority classes in the learning process. Hybrid-level approaches combine both data-level and algorithmic-level methods in provide solutions that could improve model performance on imbalanced datasets. We refer the reader to the survey articles by Tyagi et al.\cite{tyagi2020sampling}, Fernandez et al.\cite{fernandez2018algorithm}, and Ahmed et al.\cite{ahmed2023comparative} for detailed discussions on these approaches.

Our work falls into the category of Data-level approaches where we present an oversampling approach to alleviate class imbalance. Existing oversampling techniques can be divided into two categories, namely, traditional methods, and deep learning-based methods.

\paragraph{\textbf{Traditional Oversampling Methods: }} Traditional oversampling algorithms have been widely investigated and prominently applied for imbalanced classification. The most well-known oversampling technique, Synthetic Minority Over-sampling Technique (SMOTE), generates synthetic minority samples by combining existing minority data instances through linear interpolation \cite{chawla2002smote}. Building on the increased popularity of this approach, many subsequent works have been proposed with further refinements. SVMSMOTE is one such extension that combines SMOTE with principles of Support Vector Machines in order to generate minority instances that reside near the class-separating decision boundary, making the oversampling process more informative \cite{tang2008svms}. K-means SMOTE utilizes a grouping approach where they first cluster the minority class instances and then generate synthetic samples within each cluster using SMOTE, aiming to enhance the diversity in the data generation process \cite{last2017kmeanssmote}. Borderline-SMOTE \cite{han2005borderline} and ADASYN \cite{he2008adasyn} are another two techniques that aim to enhance the significance of the data generation of SMOTE by focusing on minority data instances near the decision boundary or difficult to classify, respectively. SMOTE-N and SMOTE-NC are another two extensions of SMOTE, proposed to enhance the oversampling process in categorical and continuous features effectively \cite{chawla2002smote}.

\paragraph{\textbf{Deep Learning-based Oversampling Methods: }} Deep oversampling approaches leverage deep learning models to handle synthetic data generation. DeepSMOTE \cite{dablain2022deepsmote} and GAMO \cite{mullick2019generative} have utilized Generative Adversarial Networks to generate synthetic samples for image datasets, whereas cWGAN \cite{engelmann2021conditional} employed Conditional GANs for generating synthetic data in tabular formats. Moreover, GENDA \cite{troullinou2023generative} utilized AutoEncoders (AE) for the oversampling process enabling their approach to handle imaging and time series data. These works are based on the hypothesis that the GANs and AEs excel at generating intricate, high-dimensional data and can potentially be utilized to generate minority data instances with enhanced semantic alignment. Additionally,  Ando et al. \cite{ando2017deep} proposed a Convolution Neural Network(CNN) based oversampling approach for image classification. Karunasingha et al. \cite{karunasingha2023oc} introduced OC-SMOTE-NN, an adaptive oversampling algorithm that models the SMOTE  algorithm using learnable parameters.

Our work is significantly different from these works. Firstly, GAN and AE are known to have higher training complexity and challenging training requirements such as a higher number of training epochs.  Consequently, applications utilizing those methods could suffer from model overfitting and limited generalization capabilities. Further, GAN and AE related works have limited interpretability as these models are designed to learn compact representations of data in their latent space, which may not directly correspond to interpretable features.  We use multi-layer perceptrons (MLPs) \cite{murtagh1991multilayer}, which are scalable and easy to train. Also, MPLs have higher interpretability due to their direct mapping between input and output features and are known to have the capability to approximate complex functions \cite{kruse2022multi}. Secondly, we propose utilizing a set of learnable discrete decision criteria to govern the oversampling process, which we believe is a novel perspective from all these existing works.

\section{Methodology} \label{sec:methodology}

\subsection{Preliminaries}

 Let $\mathcal{D} = \{(\mathbf{x}_i, y_i)\}_{i=1}^N$ denote a dataset, where each $\mathbf{x}_i \in \mathcal{X}$ represents a feature vector for a data instance with $f$ features (i.e., $\mathbf{x}_i \in \mathbb{R}^{f}$), and $y_i \in \mathcal{Y} = \{1, \ldots, C\}$ represents the corresponding class labels for $C$ classes. 
 
 Given a data instance $\mathbf{x}_i \in \mathcal{X}$ and $k \in \mathbb{Z}_{> 0}$, the $k$-nearest neighbors of $\mathbf{x}_i$ are represented by the set $\mathcal{N}_k(\mathbf{x}_i) \subseteq \mathcal{X} \setminus \{\mathbf{x}_i\}$, with $|\mathcal{N}_k(\mathbf{x}_i)| = k$. For any $x_j \in \mathcal{N}_k(\mathbf{x}_i)$ and $x'_j \in \mathcal{X} \setminus \mathcal{N}_k(\mathbf{x}_i)$, it holds that $y_i = y_j$ and $\textsc{Dist}(\mathbf{x}_i, \mathbf{x}_j) \leq \textsc{Dist}(\mathbf{x}_i, \mathbf{x}'_j)$, where $\textsc{Dist}(\cdot)$ represents the Euclidean distance function.
 
\subsection{Imbalanced Classification Problem}

We define the imbalanced classification problem as follows. A given classification problem is considered imbalanced if there exists at least one class $c \in \{1, \ldots C \}$ such that $|\mathcal{X}_c| \ll \textsc{Max}(|\mathcal{X}_1|, \dots |\mathcal{X}_C|)$, where $\mathcal{X}_c$ denote the set of data instances associated with that label. Typically, the imbalance ratio (IR), $\frac{\textsc{Max}(|\mathcal{X}_1|, \dots |\mathcal{X}_C|)}{\textsc{Min}(|\mathcal{X}_1|, \dots |\mathcal{X}_C|)}$ is used to quantify the level of imbalance. The objective of the imbalanced classification problem is to derive a classifier $f: \mathcal{X} \rightarrow \mathcal{Y}$ that works well with both majority and minority classes. 

\subsection{Oversampling Function}

Next, we delve into the oversampling function. An oversampling function $\sigma$ that takes the original dataset as the input and produces an augmented dataset by increasing the data instances in minority classes can be  formally expressed as follows:
\begin{equation}
    \sigma(\mathcal{D}) = \{\mathcal{D}\} \cup \{\mathcal{D}^{'}_{min}\}
\end{equation}

where $\mathcal{D}^{'}_{min}$ represent the set of synthetically generated  data instances for minority classes. Velayuthan et al. \cite{velayuthan2023revisiting} have generalized the oversampling function to operate at the level of individual data instances, defined as:
\begin{equation}
    \widetilde{x}_i = \textsc{Agg}\bigg( \{x_i\} \cup \mathcal{N}_k(x_i) \bigg)
\end{equation}
where \( \widetilde{x}_i \) represents a generated synthetic data sample derived from the data instance \( x_i \). Here, \( \textsc{Agg}(.) \) denotes an appropriate aggregation function applied to the set \( \{x_i\} \cup \mathcal{N}_k(x_i) \). This function could, for example, perform
operations such as summing or averaging over the set of points.

\subsection{Proposed Oversampling Approach}

Our aim is to design a learnable oversampling function that can generate synthetic data instances for minority classes in a data-driven manner. However, it is not well defined that which characteristics a good oversampling function should have. Therefore, we propose certain properties that a good oversampling function should possess: (1) It should be able to  \textbf{replicate intricate data dependencies} when generating synthetic samples, (2) It should be an \textbf{end-to-end architecture} that can be \textbf{jointly optimized} with the classifier, (3) It should \textbf{offer abundant variability} for oversampling process to ensure diversity among generated synthetic samples, and (4) It should be \textbf{computationally efficient}.

In this work, we propose a learnable framework for oversampling that conforms to the aforementioned properties. Our framework incorporates multiple decision criteria that collectively define the proposed oversampling function. We begin by formally defining  decision criteria.

\begin{definition} [Decision criterion] Let we denote oversampling process as a function $T$ where $\Tilde{x} = T (x)$. $T$ can be seen as a  combination of a set of decision criterions $\{DC_1, DC_2, \dots DC_m\}$ where each $DC_j$ (for $j \in [1, m]$) defines a specific aspect or condition influencing $T$. Formally, $T$ can be written as follows:

\[T(x)=\textsc{Comb}\bigg(DC_1(x), DC_2(x), \dots DC_m(x) \bigg) \]

where $m$ represents the total number of decision criteria and \( \textsc{Comb}(.) \) represents an abstract function that combines the multiples decision criterion to derive the oversampled value.
    
\end{definition}

To formalize how the decision criteria influence the oversampling process, we define the concept of Decision Criteria Mapping as follows:

\begin{definition} \label{def:dcm} [Decision Criteria Mapping] Let $\Phi_j : \mathbb{R}^{f} \rightarrow \mathbb{Z}$ be a function that maps each data instance into a integer value. Let \(\mu_j : \mathbb{Z} \rightarrow S_j\) is a bijective function that maps elements from \(\mathbb{Z}\) to \(S_j\), where \(S_j\) is a set of predefined decisions related to $j^{th}$ decision criteria. Decision Criteria Mapping is a composed function \(\psi_j : \mathbb{R}^f \rightarrow S_j\) formulated as follows:
\[
\psi_j(x) = (\mu_j \circ \Phi_j)(x); \forall x \in \mathcal{X}
\]
\end{definition}

In a nutshell, a decision criteria mapping maps each data instance selected for oversampling into a predefined decision. The collective decisions then guide the oversampling process for the data instance.

\subsection{Choices of Decision Criterias}

One interesting question to ask would be what kind of decision criterias would be important for designing the oversampling process. We identify following decision criterias to be salient to generate synthetic samples. 
\begin{enumerate}
\item \textbf{Oversampling Participation:} Determines whether a specific data instance from a minority class should be included in the oversampling process or not.
\item \textbf{K-Nearest Neighbors:} Specifies the number of nearest neighboring data instances to consider by the aggregation operation.
\item \textbf{Aggregation Function:} Defines the function used to aggregate or summarize the selected data instances.  
\end{enumerate}

The predefined decisions that we use of for these criterias will be discussed in the Section \ref{sec:predecisions}.
\subsubsection{End-to-End Architecture}

Learning decision criteria which are discrete in nature is a non-trivial task since these discrete decisions  involves non-differentiable functions. 
Typically, model training in deep learning involves computing partial derivatives of the loss function w.r.t. model parameters, and then iteratively modifying these model parameters in a way that minimizes the loss. Non-differentiable functions does not have the property of smooth gradients, making it difficult to  compute these partial derivatives, thus precluding the optimal model parameter learning. In order to overcome this challenge, we approximate these discrete decision criteria functions with continuous, differentiable functions , which enables effective back-propagation and subsequent parameter updates.

 We design each decision criteria mapping using MLPs \cite{kruse2022multi}. The main advantage of using MLPs is their ability to approximate any continuous function \cite{hornik1989multilayer}. Given $m$ decision criterias, let $f_{{\theta}_j}(.)$ be a MLP function parameterized by $\theta$ which is associated with $j$th decision criteria. The function \( f_{{\theta}_j}(.) \) maps an input $x \in \mathcal{X}$ to probabilities \( Z_{x,j} = (Z_{x,j}^{1}, Z_{x,j}^{2}, \dots, Z_{x,j}^{n}) \) for \( n \) pre-defined decisions associated with $j$th decision criteria ($j \in [1, m]$) as follows:

\begin{equation}
    Z_{x,j} = f_{{\theta}_j}(x) 
\end{equation}

Then, the decision choice denoted by $y_{x,j}$ for
data instance $x$ under each criteria $j$ is sampled from $Z_{x,j}$ using the Gumbel-Softmax \cite{jang2017categorical}  which leverages the reparametrization trick to ensure differentiability. 
\begin{equation}
    U^i_{x,j} \sim \text{Uniform}(0, 1) ; \forall i \in [1, n]
\end{equation}
\begin{equation}
   g^i{x,j} = -\log(-\log(U^i_{x,j}))  U^i_{x,j} ; \forall i \in [1, n]
\end{equation}
\begin{equation}
    y^i_{x,j} = \frac{\exp\left(\frac{z^i_{x,j} + g^i_{x,j}}{\tau}\right)}{\sum_{k=1}^{n} \exp\left(\frac{z^i_{x,j} + g^i_{x,j}}{\tau}\right)} \quad \text{with} \quad \tau \in (0, \infty) ; \forall i \in [1, n]
\end{equation}

Gumbel softmax approximates a one-hot vector $y_{x,j} = (y^1_{x,j}, y^2_{x,j}, \dots y^n_{x,j})$, and the index of the selected decision for data instance $x$ on the $j^{th}$ decision criteria, denoted by $S_{x, j}$, is determined as:

\begin{equation}
    S_{x, j} = \mu_j(\arg\max_k y^k_{x,j})
\end{equation}

Where $\mu_j(.)$ represents a bijective mapping function ( described in Definition \ref{def:dcm}). After determining the selected decisions from each decision criterion, these decisions are collectively utilized to generate synthetic samples from 
$x$.

Weights in these MLPs  are jointly optimized alongside the learnable weights in the classifier by back-propagation governed by the loss function that is used in the classifier, which is Categorical Cross Entropy loss \cite{chahkoutahi2024influence}. 

\subsubsection{Two Variants}

In this section, we explore two variants in our  learnable framework, namely, $AutoSMOTE_{self}$, and $AutoSMOTE_{cohort}$. 

\begin{figure}
    \centering
    \begin{minipage}[t]{1\textwidth}
        \centering
        \includegraphics[width=\textwidth]{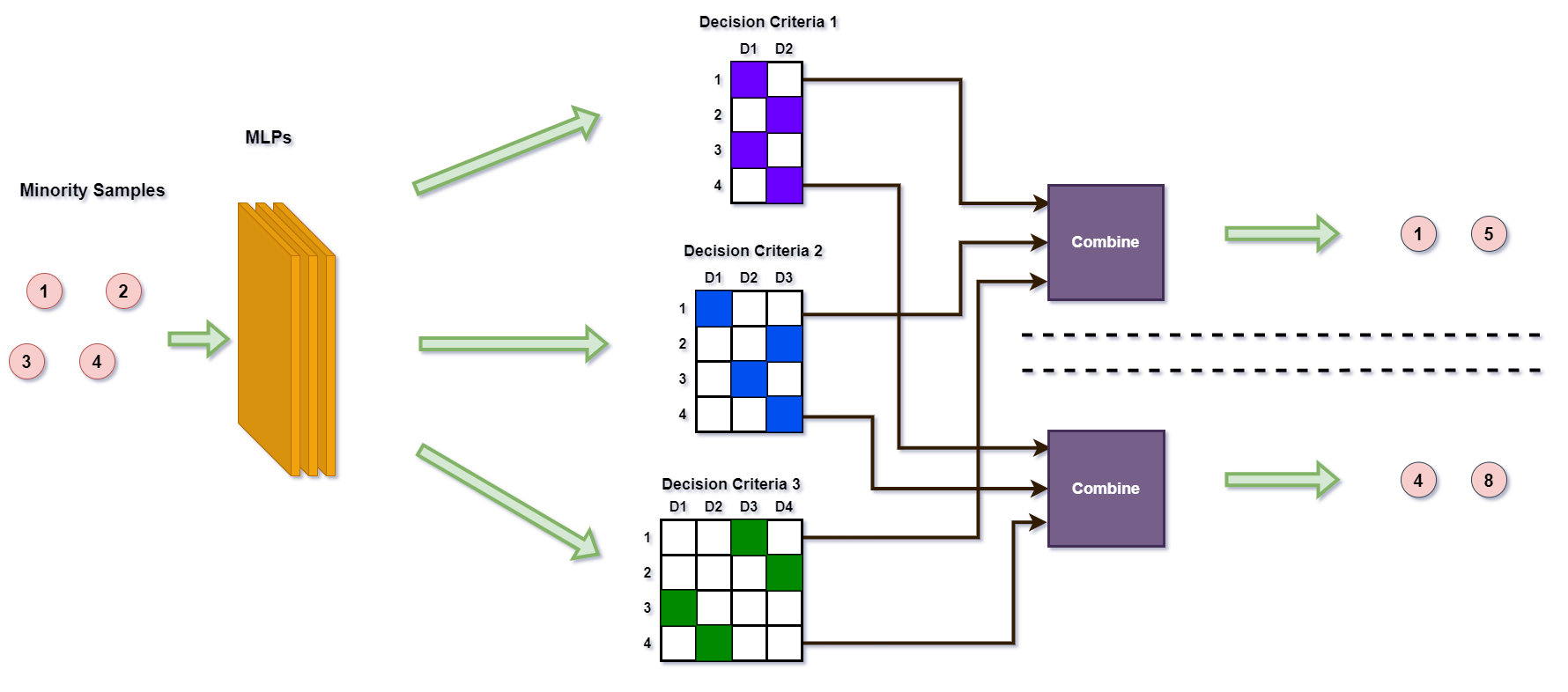}
        \subcaption{AutoSMOTE$_{self}$ }
    \end{minipage}
    \hfill
    \begin{minipage}[t]{1\textwidth}
        \centering
        \includegraphics[width=\textwidth]{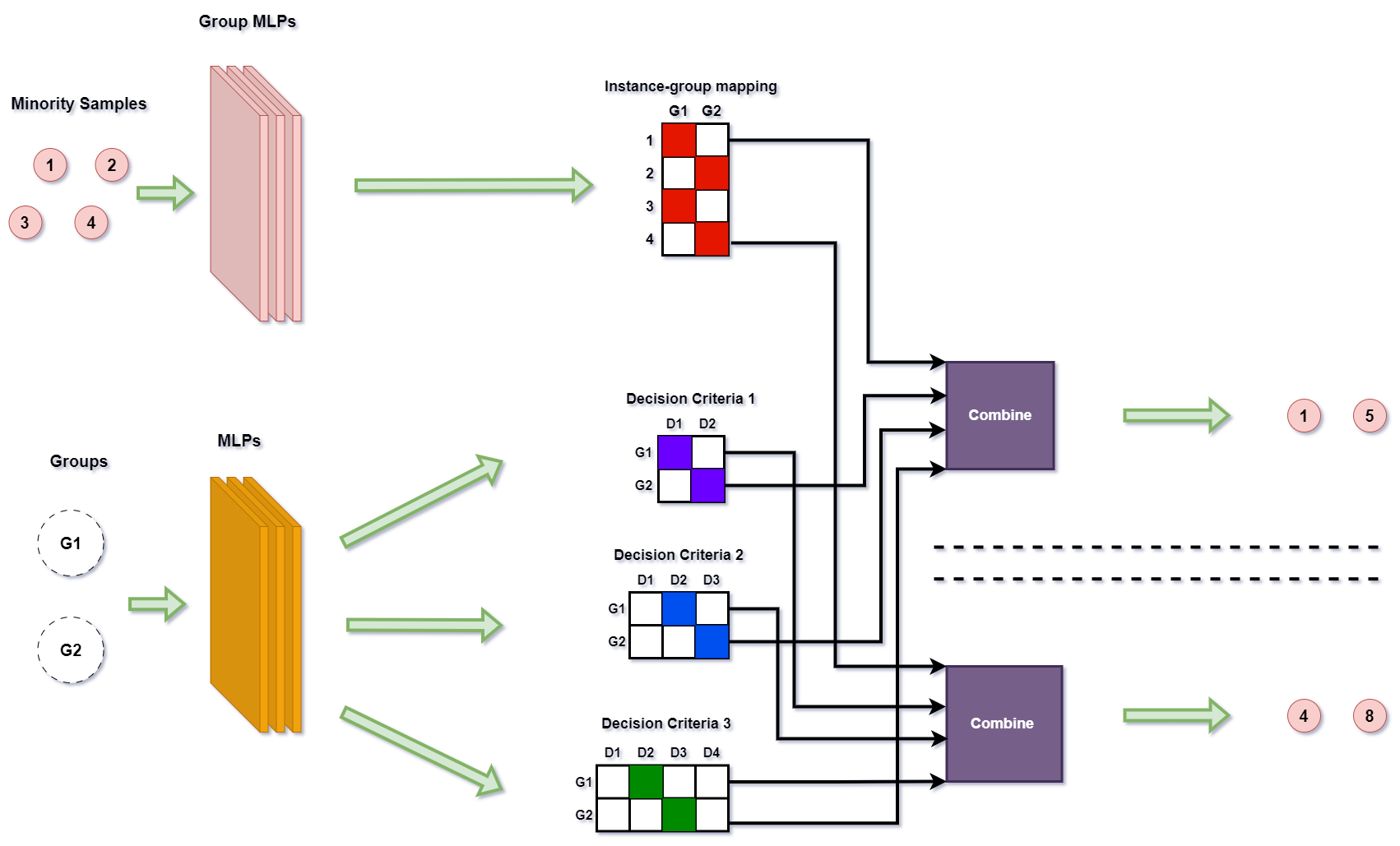}
        \subcaption{AutoSMOTE$_{cohort}$ }
    \end{minipage}
    \caption{AutoSMOTE high-level architecture: (a) AutoSMOTE$_{self}$ predicts decision criteria per each instance, (b) AutoSMOTE$_{cohort}$ considers grouping when predicting decision criteria.}
    \label{fig:architecture}
\end{figure}
\paragraph{$\mathbf{AutoSMOTE_{self}}$}In the first variant, our approach applies the defined decision criteria individually to each minority data instance. This method treats each instance independently, allowing for fine-grained oversampling decisions based on the features of each instance. 

\paragraph{$\mathbf{AutoSMOTE_{cohort}}$}The second variant involves grouping minority class instances and applying the decision criteria to each group collectively. By segmenting the minority class into groups, we aim to capture broader patterns and variations within the data, potentially improving the diversity and effectiveness of synthetic sample generation. Here, we learn the groups for each minority data instance using MLPs. The number of groups is determined as a hyperparameter.

The high-level architecture of these variants are visualized in Figure \ref{fig:architecture}.

\section{Theoretical Analysis} \label{sec:theory}

In this section, we analyze the functional approximation capacity and generalization capabilities of AutoSMOTE.

MLPs are known have the capability of approximating any continuous function as stated by the Universal approximation theorem \cite{hornik1989multilayer,hornik1991approximation} which is deemed to be an ideal choice to model oversampling function. In contrast, AutoSMOTE defines the oversampling function as a composition of predefined decision criterias which interently makes it functional approximation capability limited compared of MLPs. Therefore, one might question the advantage of our approach over more expressive MLPs for this learning task. We seek to answer this question through our theoretical analysis. We start by providing formal definition of universal approximation theorem.

\begin{theorem} (Universal Approximation Theorem)
    For \( G \subset \mathbb{R}^n \),  we define \( R(G) \) as the set of all continuous functions from $G$ to $\mathbb{R}$: $R(G) = \{ f : G \to \mathbb{R} \mid f \text{ is continuous} \}$. Then, for any $f \in R(G)$  and for any \(\epsilon > 0\), there exists a multi-layer perceptron \( \phi \) with a single hidden layer, a finite number of neurons such that:

\[
\sup_{x \in G} | f(x) - \phi(x) | < \epsilon.
\]

\end{theorem}

In a nutshell, MLPs can approximate any continuous function on a compact subset of $\mathbb{R}^n$ to arbitrary accuracy. Consequently, they should be well-suited to effectively model an ideal oversampling function for a given dataset. However, finding such ideal function is non-trivial as this compact subset of functions can be expansive, making the optimization process arduous and vulnerable to overfitting. Moreover, the expansive search space could negatively impact the generalization ability of the model as the model parameters are overly tuned for the training set, leading to suboptimal performance on the testing set and higher testing error. In oder to theoretically analyze this, we use Vapnik–Chervonenkis (VC) dimension \cite{abu1989vapnik} which is a commonly employed metric for analyzing the capacity of machine learning models. 

\begin{definition}[VC Dimension]
Let $\mathcal{H}$ be a hypothesis class of binary functions $h: \Omega \to \{0, 1\}$, where $\Omega \in\mathbb{R}^d$ is a set. The VC dimension $VC_{dim}(\mathcal{H})$ of $\mathcal{H}$ is the supremum of the cardinality of any finite set $S \subseteq \Omega$ that can be shattered (i.e. correctly classified) by $\mathcal{H}$.
\end{definition}

A lower VC dimension naturally refers to a a less complex model with a reduced overfitting risk. With a lower VC dimension, the model is more likely to generalize better to the testing set. We define the generalization bound w.r.t. VC dimension as follows.

\begin{definition} \label{def:ebvc} [Error Bound of VC Dimension] Let $\mathcal{H}$ denote a hypothesis class with VC dimension $VC_{dim}(\mathcal{H})$.  The generalization error bound $\epsilon_{\text{gen}}$  is upper bounded as follows:

\[ \epsilon_{\text{gen}} \leq \epsilon_{\text{train}} + O\bigg( \sqrt{\frac{VC_{dim}(\mathcal{H}) \log(N)}{N}}\bigg) \]

where $\epsilon_{\text{train}}$ is the empirical error  of $\mathcal{H}$, and $N$ is the training sample size. 
\end{definition}

As per the above definition, generalization error bound of the model would generally increase with the VC dimension, potentially leading to decreased testing set performance. Based on this observation, we derive the following theorem.

\begin{theorem}   \label{thm:thmgen} 
Let $\mathcal{H}_{\text{AutoSMOTE}_{cohort}}, \mathcal{H}_{\text{AutoSMOTE}_{self}},$ and $\mathcal{H}_{\text{MLP}}$ denote the hypothesis classes of $AutoSMOTE_{cohort}$, $AutoSMOTE_{self}$, and MLP models, respectively. Assume the following conditions hold:
\begin{enumerate}
    \item The empirical training errors \(\epsilon_{\text{train}}(\mathcal{H}_{\text{AutoSMOTE}_{\text{cohort}}})\), \(\epsilon_{\text{train}}(\mathcal{H}_{\text{AutoSMOTE}_{\text{self}}})\), and \(\epsilon_{\text{train}}(\mathcal{H}_{\text{MLP}})\) are approximately equal.
    \item The number of training samples \(N\) is fixed for all models.
\end{enumerate}

Given \(\epsilon_{\text{gen}}(\mathcal{H})\) as the generalization error of the hypothesis class \(\mathcal{H}\), the following inequality holds:
\[ \epsilon_{\text{gen}}(\mathcal{H}_{\text{AutoSMOTE}_{cohort}}) < \epsilon_{\text{gen}}(\mathcal{H}_{\text{AutoSMOTE}_{self}}) < \epsilon_{\text{gen}}(\mathcal{H}_{\text{MLP}}). \]
\end{theorem}

\begin{proof}
    From the design, we know that $\mathcal{H}_{\text{AutoSMOTE}_{cohort}} \subset \mathcal{H}_{\text{AutoSMOTE}_{self}} \subset \mathcal{H}_{\text{MLP}}$.  Therefore, it is reasonable to expect that their VC-dimensions would follow the same order: $\text{VC}_{dim}(\mathcal{H}_{\text{AutoSMOTE}_{cohort}}) < \text{VC}_{dim}(\mathcal{H}_{\text{AutoSMOTE}_{self}}) < \text{VC}_{dim}(\mathcal{H}_{\text{MLP}})$. 
    
    Now, let's focus on Definition \ref{def:ebvc} which is on the generalization error bound of VC-dimension. Under the given assumptions, we can approximate the following relationship:
    \[
\epsilon_{\text{gen}}(\mathcal{H}) \propto \text{VC}_{\text{dim}}(\mathcal{H})
\]

This approximation would leads to the inequality: $\epsilon_{\text{gen}}(\mathcal{H}_{\text{AutoSMOTE}_{cohort}}) < \epsilon_{\text{gen}}(\mathcal{H}_{\text{AutoSMOTE}_{self}}) < \epsilon_{\text{gen}}(\mathcal{H}_{\text{MLP}})$. This completes the proof.
\end{proof}

As established by the above theorem,  AutoSMOTE variants exhibit comparatively lower generalization error bounds w.r.t. MLPs, leading to better performance on testing data. Moreover, we empirically validate these theoretical findings in Section \ref{sec:generalization_perfformance}.

\section{Experimental Design} \label{sec:expriment}

Through our experimental analysis, we seek to answer the following research questions.

\begin{itemize}
    \item \textbf{RQ1: }How well AutoSMOTE variants outperform state-of-the-art oversampling algorithms in the imbalanced classification task?
    \item \textbf{RQ2: }What are the generalization capabilities of AutoSMOTE variants?
    \item \textbf{RQ3: }How do different decision criterias contribute to AutoSMOTE’s overall performance?
    \item \textbf{R4: }How does the training efficiency of AutoSMOTE variants compare with existing oversampling algorithms?
\end{itemize}

\subsection{Datasets}
We employ 8 datasets related for diverse fields in our experiments. They are summarized as follows:

\begin{itemize}
    \item Diabetes \cite{smith1988using} - A binary dataset based on diagnostic measurements of diabetes. It consists of 768 data instances with 8 features each.
    \item Page-blocks \cite{fernandez2008study} - A multi-class dataset related to document analysis. It has of 5473 data instances with 10 features each.
    \item Glass \cite{dong2011new} - A multi-class dataset on identification of different glass types. It consists of 214 data instances with 9 features each.
    \item Wisconsin \cite{street1993nuclear} - A binary dataset about patient traits associated with breast cancer. It consists of 569 data instances with 30 features each.
    \item Thyroid \cite{pang2019deep} - A binary dataset that contains data in relation to thyroid disorders. It consists of 7200 data instances with 21 features each.
    \item  Kc1 \cite{menzies2004assessing} - A binary dataset on software defect prediction. It has 2109 data instance with 21 features each.
    \item Yeast \cite{read2011classifier} - A multi-class dataset related to microbiology. It contains 2417 data instances with 103 attributes each.
    \item Ads \cite{das2018sparse} - A binary dataset on predicting possible internet advertisements. It contains 3279 data instances with 1558 attributes each.
\end{itemize}

\subsection{Experimental Setups and Baselines}

We follow a standard experimental setup commonly used in machine learning experiments. It includes a data prepossessing step used to normalize dataset features and fill dummy values for missing features. We employ a stratified training testing split under multiple model executions and record average and standard deviation of the evaluation metrics. 

For the baselines for imbalanced classification task, we compare our model with both traditional and deep learning-based oversampling methods. These traditional methods include SMOTE \cite{chawla2002smote}, SVMSMOTE \cite{tang2008svms}, BorderlineSMOTE \cite{han2005borderline}, ADASYN \cite{he2008adasyn}, and SMOTE-N \cite{chawla2002smote}. For deep learning approaches, we employ oversamplers based on MLPs, GANs and Variational AEs.

\subsection{Evaluation Metrics}

We use precision, recall and F1-score \cite{wardhani2019cross} to evaluate our models and baselines. The equations for these metrics are as follows:

\[ \text{Precision} = \frac{TP}{TP + FP}
 \]
\[ \text{Recall} = \frac{TP}{TP + FN} \]
\[ \text{F1 Score} = 2 \times \frac{\text{Precision} \times\text{Recall}}{\text{Precision} + \text{Recall}}
 \]

where TP, FP, and FN refers to true positive, false positive, and false negative counts, respectively.

\subsection{Model Hyper-parameters}

We perform training testing split with 80:20 ratio where we run each model 10 times under different seeds. We use a MLP with 1 hidden layer consisting 64 neurons as our classifier.  Further, we employ 0.05 learning rate with Adam optimizer \cite{kingma2014adam} for model training. We train our models and baselines for 200 epochs, with the exception of GAN and VAE based oversamplers, which are trained for 10,000 epochs. This extended training duration is necessary to achieve comparable performance for these models due to their inherent complexities and nature.

Batch size for each dataset is critical for model training due to different imbalanced ratios in these datasets. We select an appropriate batch size from \{500, 2500, 5000\} for each dataset which provides an adequate number of minority samples for the oversampling process. We also observe that the classification performance of oversampling algorithms heavily depends on the number of selected neighbors. Therefore, we run each model across 2 to 6 nearest neighbors and record the best performance metric, enabling fair comparison. Further, $AutoSMOTE_{cohort}$ requires determining the optimal number of groups into which minority samples should be grouped. In order to do that, we perform a search across 1 to 7 groups and record the best result.

\subsection{Pre-defined Decisions of AutoSMOTE} \label{sec:predecisions}

AutoSMOTE requires set of pre-defined decisions under each decision criteria. For Oversampling Participation, we would learn a binary variable for each minority instance that would determine whether it would be utilized to generate synthetic samples (1) or not (0). If a minority instance is decided to be not utilized (0), then there will be no synthetic sample generated using that instance.
 We learn a number between 1 to 6 for each minority instance which is then employed as the number of nearest neighbors for that instance in the oversampling process. For the aggregation function, we select a function for each minority instance in a learnable manner. The set of functions available for selection include linear interpolation (used in SMOTE), minimum, maximum, sum, average, and weighted average functions. The implementation aspects of these functions are provided in Appendix \ref{sec:agg}.

\subsection{System Resources and Implementation Details} 

We use following libraries and programming languages for our implementation: Python\cite{van2007python} , Scikit-learn \cite{bisong2019introduction}, Pytorch \cite{imambi2021pytorch}, Pandas \cite{mckinney2015pandas}, and Numpy \cite{bressert2012scipy}. All the experiments are executed on a computer with Core i7 CPU, 16 GB RAM and 8GB GPU.

\section{Results and Discussion}\label{sec:result}
\subsubsection{Imbalanced Classification Performance (RQ1)}

In this section, we compare the classification performance of our model with the baselines. Table \ref{tab:classification} depicts the results for each dataset under each evaluation metric. 

\renewcommand{\arraystretch}{.7}
\setlength{\tabcolsep}{1pt}

\begin{longtable}{cc|ccc}

\caption{Performance comparison of oversampling methods. Standard deviation for each metric is given in brackets. The best results are highlighted in \textbf{bold}. Note: Results for some datasets under the ADASYN method are not reported due to errors encountered due to insufficient minority samples.} \label{tab:comparison} \\

\toprule
\textbf{Dataset} & \textbf{Method} & \textbf{Precision}&\textbf{Recall}&\textbf{F1-score} \\
\midrule
\endfirsthead

\multicolumn{3}{c}%
{{\bfseries Table \thetable\ Continued from previous page}} \\
\toprule
\textbf{Dataset} & \textbf{Method} & \textbf{Precision} & \textbf{Recall} & \textbf{F1-score}\\
\midrule
\endhead

\midrule \multicolumn{3}{r}{{Continued on next page}} \\
\endfoot

\bottomrule
\endlastfoot

\multirow{3}{*}{Diabetes}
  & SMOTE & 73.21 ($\pm$3.91)&74.24 ($\pm$4.21)&72.78 ($\pm$4.36)  \\
  & SVMSMOTE & 73.50 ($\pm$3.24)&\textbf{74.45 ($\pm$3.42)}&\textbf{72.82 ($\pm$3.52)}\\
  & BorderlineSMOTE &72.88 ($\pm$3.60)&74.06 ($\pm$3.76)&72.21 ($\pm$4.10)\\
  & ADASYN & -& -& -  \\
  & SMOTE-N &70.32 ($\pm$4.39)&71.37 ($\pm$4.65)&69.50 ($\pm$4.98) \\
  & MLP-Oversampler &73.82 ($\pm$5.06)&71.72 ($\pm$4.71)&72.03 ($\pm$4.50) \\
  & VAE-Oversampler &72.73 ($\pm4.65$)&72.14 ($\pm4.69$)& 71.98 ($\pm4.56$) \\
  & GAN-Oversampler & 72.59($\pm4.03$)&71.94($\pm3.61$)& 71.98($\pm3.73$) \\
  & AutoSMOTE$_{Self}$ & 73.42 ($\pm$5.01)&71.64 ($\pm$4.87)&72.10 ($\pm$4.76) \\
  & AutoSMOTE$_{cohort}$ & \textbf{74.12 ($\pm$5.01)}&72.53 ($\pm$4.87)&72.74 ($\pm$4.76) \\

\midrule
\multirow{3}{*}{Page-blocks}
  & SMOTE &54.78 ($\pm$2.18)&\textbf{91.53 ($\pm$1.69)}&64.53 ($\pm$2.48) \\
  & SVMSMOTE & 55.26 ($\pm$2.97)&90.93 ($\pm$2.11)&65.22 ($\pm$3.26) \\
  & BorderlineSMOTE &51.86 ($\pm$2.73)&90.81 ($\pm$2.26)&61.81 ($\pm$3.12)\\
  & ADASYN &71.63 ($\pm$2.96)&68.81 ($\pm$6.16)&59.57 ($\pm$6.92)  \\
  & SMOTE-N &55.98 ($\pm$2.62)&88.15 ($\pm$3.46)&65.3 ($\pm$2.57) \\
  & MLP-Oversampler &59.89 ($\pm$20.00)&38.49 ($\pm$6.04)&42.41 ($\pm$8.60) \\
  & VAE-Oversampler & 82.13 ($\pm3.62$)&66.27 ($\pm4.82$)& 71.48 ($\pm4.32$) \\
  & GAN-Oversampler & \textbf{82.93($\pm4.23$)}&69.44($\pm6.42$)& \textbf{74.00($\pm5.97$)} \\
  &  AutoSMOTE$_{Self}$ &81.25 ($\pm$11.38)&58.84 ($\pm$20.42)&65.17 ($\pm$12.11)\\
  & AutoSMOTE$_{cohort}$ &60.61 ($\pm$11.38)&72.42 ($\pm$20.42)&60.12 ($\pm$12.11)\\
\midrule
\multirow{3}{*}{Glass}
  & SMOTE & 68.49 ($\pm$8.55)&\textbf{69.64 ($\pm$8.08)}&66.88 ($\pm$8.34) \\
  & SVMSMOTE &69.23 ($\pm$10.04)&66.30 ($\pm$6.99)&65.59 ($\pm$8.34)  \\
  & BorderlineSMOTE &68.90 ($\pm$10.03)&66.67 ($\pm$10.67)&65.50 ($\pm$10.88) \\
  & ADASYN &65.13 ($\pm$8.71)&64.63 ($\pm$6.47)&63.40 ($\pm$7.30) \\
  & SMOTE-N &70.25 ($\pm$11.15)&66.03 ($\pm$10.64)&65.82 ($\pm$10.92) \\
  & MLP-Oversampler &58.12 ($\pm$10.55)&57.60 ($\pm$10.18)&56.60 ($\pm$10.00) \\
  & VAE-Oversampler & 58.33($\pm11.50$)&57.48($\pm11.49$)&56.49($\pm11.38$) \\
  & GAN-Oversampler & 63.78($\pm9.28$)&60.33($\pm9.55$)& 60.17($\pm9.42$) \\
  & AutoSMOTE$_{Self}$ &67.29 ($\pm$11.68)&62.99 ($\pm$9.89)&62.89 ($\pm$10.22) \\
  &  AutoSMOTE$_{cohort}$ &\textbf{70.95 ($\pm$11.68)}&68.58 ($\pm$9.89)&\textbf{67.87 ($\pm$10.22)}\\
\midrule
\multirow{3}{*}{Wisconsin}
  & SMOTE & 97.84 ($\pm$1.02)&97.86 ($\pm$0.88)&97.83 ($\pm$0.89)  \\
  & SVMSMOTE &97.08 ($\pm$1.38)&97.59 ($\pm$1.01)&97.29 ($\pm$1.19)  \\
  & BorderlineSMOTE &96.51 ($\pm$1.22)&97.24 ($\pm$0.98)&96.83 ($\pm$1.09) \\
  & ADASYN &96.25 ($\pm$1.79)&97.03 ($\pm$1.15)&96.56 ($\pm$1.50) \\
  & SMOTE-N &96.40 ($\pm$1.56)&96.95 ($\pm$1.27)&96.63 ($\pm$1.40) \\
  & MLP-Oversampler &\textbf{98.00 ($\pm$0.93)}&\textbf{97.88 ($\pm$0.91)}&\textbf{97.92 ($\pm$0.87)} \\
  & VAE-Oversampler & 97.14 ($\pm1.32$)&97.15 ($\pm0.71$)& 97.09 ($\pm0.90$) \\
  & GAN-Oversampler & 97.89 ($\pm1.11$)& 97.81($\pm0.98$)& 97.83 ($\pm0.99$) \\
  & AutoSMOTE$_{Self}$ &\textbf{98.00 ($\pm$0.93)}&\textbf{97.88 ($\pm$0.91)}&\textbf{97.92 ($\pm$0.87)}\\
  & AutoSMOTE$_{cohort}$ &\textbf{98.00 ($\pm$0.93)}&\textbf{97.88 ($\pm$0.91)}&\textbf{97.92 ($\pm$0.87)} \\
\midrule
\multirow{3}{*}{Thyroid}
  & SMOTE &86.41 ($\pm$4.98)&94.29 ($\pm$3.36)&\textbf{89.45 ($\pm$2.88) } \\
  & SVMSMOTE &82.85 ($\pm$3.42)&88.06 ($\pm$6.17)&84.61 ($\pm$2.09)  \\
  & BorderlineSMOTE &84.23 ($\pm$4.62)&93.38 ($\pm$4.18)&87.63 ($\pm$2.64)\\
  & ADASYN &85.21 ($\pm$7.85)&\textbf{95.55 ($\pm$2.10)}&88.79 ($\pm$7.49) \\
  & SMOTE-N &74.87 ($\pm$3.43)&86.16 ($\pm$3.08)&78.95 ($\pm$3.11)\\
  & MLP-Oversampler &91.56 ($\pm$1.15)&74.6 ($\pm$4.13)&80.34 ($\pm$3.00)\\
  & VAE-Oversampler & \textbf{93.08 ($\pm0.94$)}&79.71 ($\pm4.26$)& 84.78 ($\pm3.01$) \\
  & GAN-Oversampler & 87.21($\pm9.05$)&76.81($\pm6.13$)& 80.34($\pm6.53$) \\
  & AutoSMOTE$_{Self}$ &92.03 ($\pm$8.78)&85.77 ($\pm$6.53)&88.26 ($\pm$5.12)\\
  &AutoSMOTE$_{cohort}$ &83.43 ($\pm$8.78)&91.10 ($\pm$6.53)&84.93 ($\pm$5.12) \\
\midrule
\multirow{3}{*}{Kc1}
  & SMOTE & 62.00 ($\pm$1.87)&\textbf{71.60 ($\pm$3.07)}&60.90 ($\pm$2.71) \\
  & SVMSMOTE &63.54 ($\pm$3.02)&70.28 ($\pm$3.43)&\textbf{64.50 ($\pm$3.36)}\\
  & BorderlineSMOTE &61.91 ($\pm$1.76)&71.29 ($\pm$2.70)&61.07 ($\pm$2.75) \\
  & ADASYN &61.49 ($\pm$1.79)&71.58 ($\pm$3.24)&58.81 ($\pm$2.73)  \\
  & SMOTE-N & 60.71 ($\pm$3.48)&63.59 ($\pm$4.48)&61.21 ($\pm$3.47) \\
  & MLP-Oversampler &\textbf{78.82 ($\pm$4.49)}&60.12 ($\pm$2.19)&62.83 ($\pm$2.91)\\
  & VAE-Oversampler & 76.23 ($\pm6.39$)&59.55 ($\pm2.90$)& 61.95 ($\pm3.80$) \\
  & GAN-Oversampler & 72.30($\pm7.30$)&59.33($\pm2.54$)& 61.46($\pm3.50$) \\
  & AutoSMOTE$_{Self}$ & 61.65 ($\pm$5.19)&63.46 ($\pm$6.74)&62.32 ($\pm$4.19) \\
  & AutoSMOTE$_{cohort}$ & 63.56 ($\pm$5.19)&67.18 ($\pm$6.74)&62.71 ($\pm$4.19)\\
\midrule
\multirow{3}{*}{Yeast}
  & SMOTE & 50.36 ($\pm$5.15)&57.05 ($\pm$4.33)&50.74 ($\pm$4.09)  \\
  & SVMSMOTE &50.61 ($\pm$5.57)&56.83 ($\pm$4.23)&51.10 ($\pm$4.26)  \\
  & BorderlineSMOTE &50.70 ($\pm$4.83)&55.76 ($\pm$4.50)&49.59 ($\pm$3.98) \\
  & ADASYN & -& -& -  \\
  & SMOTE-N &47.25 ($\pm$3.91)&55.54 ($\pm$4.17)&48.62 ($\pm$3.81) \\
  & MLP-Oversampler &54.96 ($\pm$5.93)&51.73 ($\pm$7.08)&51.78 ($\pm$6.57) \\
  & VAE-Oversampler & \textbf{76.23 ($\pm6.39$)}& \textbf{59.55 ($\pm2.90$)}& \textbf{61.95 ($\pm3.80$)} \\
  & GAN-Oversampler & 56.27($\pm6.62$)&54.47($\pm4.60$)& 53.60($\pm5.02$) \\
  & AutoSMOTE$_{Self}$ &52.28 ($\pm$8.11)&55.68 ($\pm$5.06)&52.08 ($\pm$5.39) \\
  & AutoSMOTE$_{cohort}$ &56.59 ($\pm$8.11)&53.81 ($\pm$5.06)&52.65 ($\pm$5.39)\\
\midrule
\multirow{3}{*}{Ads}
  & SMOTE & 89.29 ($\pm$2.02)&92.87 ($\pm$1.43)&90.93 ($\pm$1.57)  \\
  & SVMSMOTE & 90.09 ($\pm$2.28)&93.04 ($\pm$1.06)&91.45 ($\pm$1.58) \\
  & BorderlineSMOTE &89.77 ($\pm$2.33)&\textbf{93.13 ($\pm$1.37)}&91.30 ($\pm$1.70)\\
  & ADASYN &88.28 ($\pm$2.17)&93.08 ($\pm$1.21)&90.41 ($\pm$1.54)\\
  & SMOTE-N &79.45 ($\pm$1.30)&88.92 ($\pm$1.44)&83.00 ($\pm$1.31) \\
  & MLP-Oversampler &\textbf{93.74 ($\pm$1.48)}&92.91 ($\pm$1.60)&\textbf{93.30 ($\pm$1.33)} \\
  & VAE-Oversampler & 92.57 ($\pm1.41$)&92.39 ($\pm1.83$)& 92.46 ($\pm1.43$) \\
  & GAN-Oversampler & 92.99($\pm1.34$)&92.57($\pm1.70$)& 92.75($\pm1.28$) \\
  & AutoSMOTE$_{Self}$ & 90.19 ($\pm$1.77)&92.67 ($\pm$1.97)&91.33 ($\pm$1.48)\\
  &AutoSMOTE$_{cohort}$ &91.16 ($\pm$1.77)&92.31 ($\pm$1.97)&91.68 ($\pm$1.48)\\
  \label{tab:classification}
\end{longtable}

In Table \ref{tab:ranking}, we provide the average ranking of these models in order the compare their overall performance . Note that a lower rank signifies better performance. 

\begin{table}[h!]
    \centering
    \caption{Average ranking across all datasets for each metric, alongside the overall ranking. The best results are highlighted in \textbf{bold}.}
    \begin{tabular}{lcccc}
        \toprule
        \textbf{Method} & \textbf{Precision    } & \textbf{Recall    } & \textbf{F1-score    } & \textbf{Overall Rank} \\
        \midrule
        SMOTE & 6.4 & \textbf{2.3} & 4.9 & 4.5\\
        SVMSMOTE & 6.0  & 3.5 & 4.3 & 4.6\\
        BorderlineSMOTE & 6.9 & 3.4 & 6.4 & 5.5 \\
        ADASYN & 7.5 & 4.5 & 7.7 & 6.6 \\
        SMOTE-N & 8.3 & 7.0 & 7.5 & 7.6    \\
        MLP-Oversampler & 3.5 & 7.3 & 5.3 & 5.3\\
        VAE-Oversampler & 3.9 & 7.1 & 5.1 & 5.4\\
        GAN-Oversampler  & 4.1 & 7.3& 4.8 & 5.4\\
        AutoSMOTE$_{self}$ & 4.3 & 6.3 & 4.4 & 5.0\\
        AutoSMOTE$_{cohort}$ & \textbf{3.3} & 4.6 & \textbf{3.5} & \textbf{3.8}\\
        \bottomrule
    \end{tabular}
    \label{tab:ranking}
\end{table}

According to results, AutoSMOTE$_{cohort}$ emerges the best performing model by achieving the best overall rank. It also achieves the top rank in Precision and F1-score metrics, and performs well in Recall. AutoSMOTE$_{self}$ also depicts competitive performance by achieving the second best overall ranking in deep learning-based oversampling methods.

\subsection{Generalization Capabilities of our Variants (RQ2)} \label{sec:generalization_perfformance}

In this section, we empirically analyze the testing set generalizability of AutoSMOTE variants w.r.t. MLP. Figure \ref{fig:training_error} shows the training error of AutoSMOTE variants alongside the MLP model for Glass and Page-blocks datasets. MLP model gradually reaches to zero training error, which typically indicates overfitting, suggesting a higher VC dimension as it tries to perfectly fit onto the training data. On the other hand, AutoSMOTE variants reach to moderate levels of training error suggesting that they are making balanced progress in learning from the training data without overfitting. This also aligns with the expected VC dimensions of our variants.

\begin{figure}[H]
    \centering
    \begin{subfigure}{1\textwidth}  
        \centering
        \includegraphics[width=\linewidth]{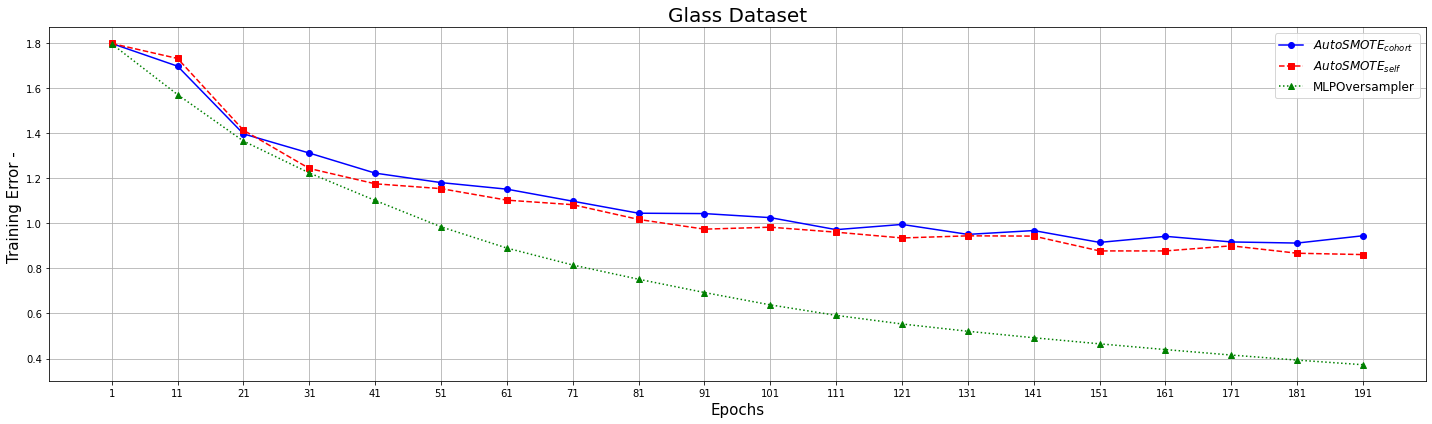}
    \end{subfigure}
    \hfill
    \begin{subfigure}{1\textwidth}  
        \centering
        \includegraphics[width=\linewidth]{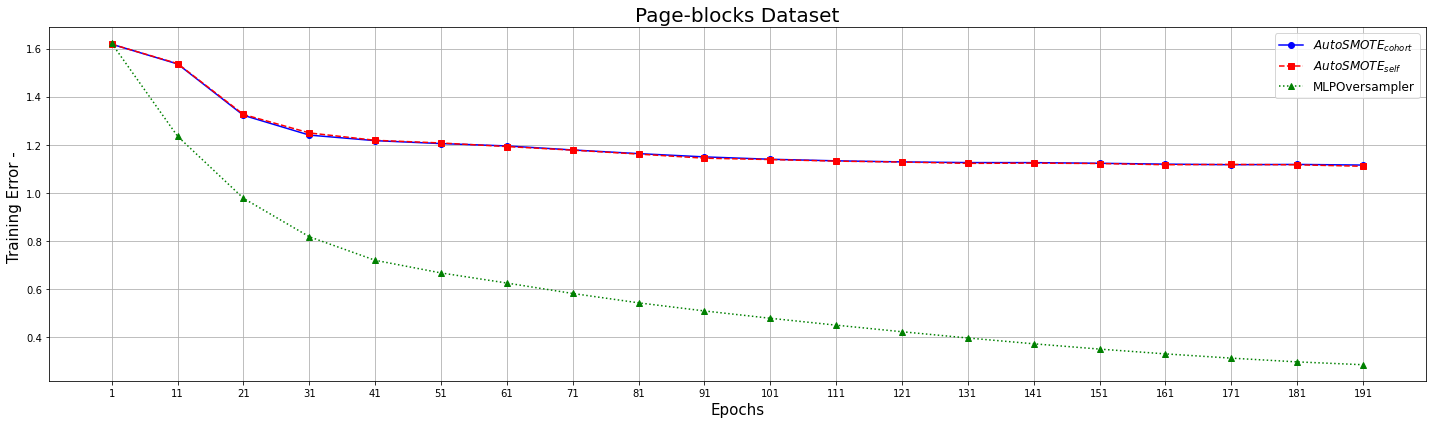}
    \end{subfigure}
    
    \caption{Training Error Comparison of AutoSMOTE with MLP-Oversampler}
    \label{fig:training_error}
\end{figure}

Figure \ref{fig:test_performance} depicts the test set performance of these three models. AutoSMOTE variants consistently outperform the MLP model as they have higher generalization capabilities. Also, $AutoSMOTE_{cohort}$ performs comparatively better than $AutoSMOTE_{self}$, aligning with their VC dimensions and generalization error bounds. This validates the applicability of Theorem \ref{thm:thmgen} in real-world datasets.

\begin{figure}[H]
    \centering
    \begin{minipage}[b]{0.48\textwidth} 
        \centering
        \includegraphics[width=\textwidth]{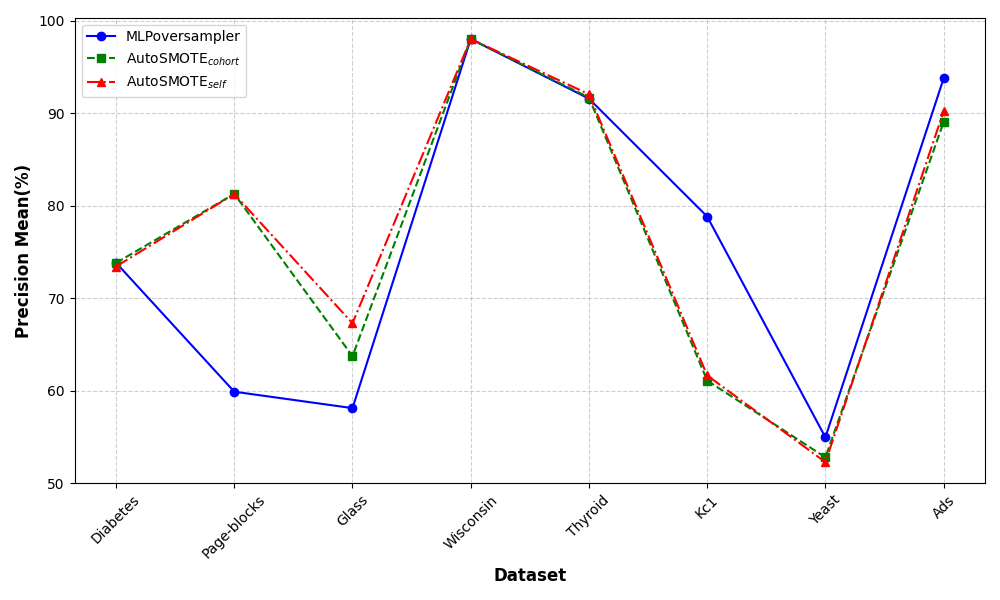}
    \end{minipage}
    \hfill
    \begin{minipage}[b]{0.48\textwidth} 
        \centering
        \includegraphics[width=\textwidth]{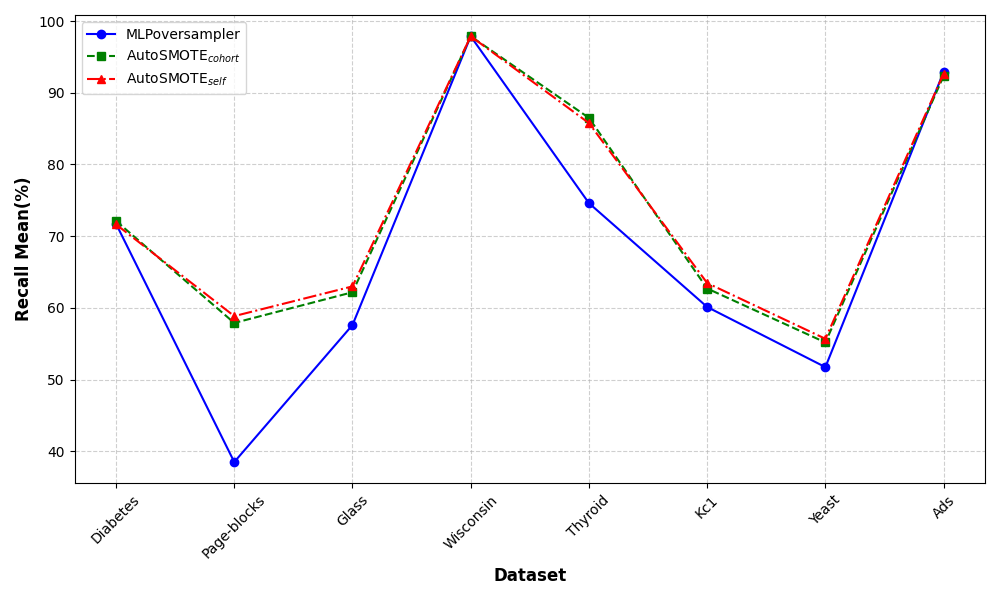}
    \end{minipage}
    
    \vskip 10pt 
    
    \begin{minipage}[b]{0.6\textwidth} 
        \centering
        \includegraphics[width=\textwidth]{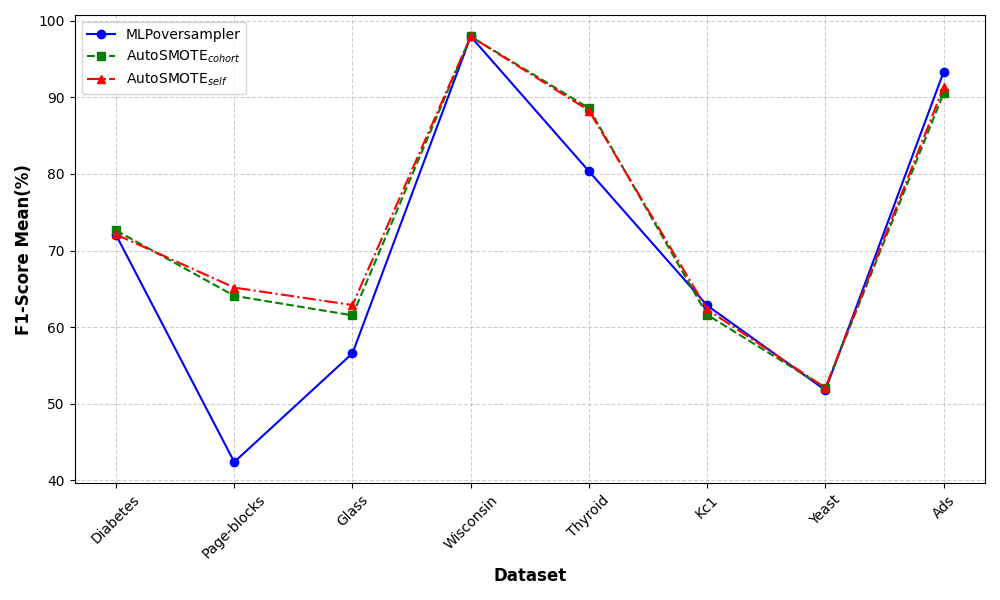}
    \end{minipage}
    
    \caption{Testing Performance Comparison of AutoSMOTE with MLP-Oversampler.}
    \label{fig:test_performance}
\end{figure}

\subsection{Ablation Study (RQ3)}

In this experiment, we analyze the impact of each decision criteria on AutoSMOTE's performance. To demonstrate this we derive three sub variants of our model by eliminating decision criterion in turn.
\begin{itemize}
  \item \textbf{AutoSMOTE Without DC1 : } We exclude the oversampling participation criteria. In this variant, all minority samples are employed in the oversampling process.
  \item  \textbf{AutoSMOTE Without DC2 : } We exclude the learnable k-nearest neighbor number selection. In this variant, the number of k-nearest neighbors will be selected as per the model hyper-parameter. 
  \item  \textbf{AutoSMOTE Without DC3 : } We exclude the learnable aggregation function selection. In this variant, all minority samples would use linear interpolation function (used in SMOTE) as their aggregation mechanism.
\end{itemize}

  \begin{table}[H]
    \centering
    \caption{Ablation study for decision criteria}
    \begin{tabular}{llcc}
        \toprule
        & & $\mathbf{AutoSMOTE_{self}}$ & $\mathbf{AutoSMOTE_{cohort}}$ \\
        \midrule
        & Baseline & 62.89 (±10.22) &  67.87 (±10.22) \\
        \multirow{3}{*}{Glass} & Without DC1 & 66.28 (±10.94) & 64.90 (±10.44) \\
                               & Without DC2 & 60.96 (±10.29) & 60.08 (±11.99) \\
                               & Without DC3 & 61.25 (±10.19) & 60.64 (±9.85) \\
        \midrule
        & Baseline & 72.10 (±4.76) & 72.74 (±4.76) \\
        \multirow{3}{*}{Diabetes} & Without DC1 & 71.73 (±4.86) & 71.45 (±3.61) \\
                                  & Without DC2 & 71.31 (±3.72) & 71.27 (±4.06) \\
                                  & Without DC3 & 71.71 (±3.51) & 71.64 (±4.02) \\
        \midrule
        & Baseline & 52.08 (±5.39) & 52.65 (±5.39) \\
        \multirow{3}{*}{Yeast} & Without DC1 & 48.28 (±3.57) & 48.27 (±4.22) \\
                               & Without DC2 & 52.10 (±4.57) & 52.56 (±4.79) \\
                               & Without DC3 & 52.71 (±4.61) & 51.94 (±5.75) \\
        \bottomrule
    \end{tabular}
    \label{table:ablation}
\end{table}

  Table \ref{table:ablation} portray the results
of our sub variants compared alongside the baseline. We can see that different datasets different dependencies on different decision criteria. For example, Yeast dataset shows significant performance decline when the oversampling participation is excluded, but shows negligible impact with the removal of other two criteria.

\subsection{Runtime Efficiency Comparision (RQ4)}

\begin{figure}[H]
     \centering
     \begin{subfigure}{0.8\textwidth} 
         \centering
         \includegraphics[width=\textwidth]{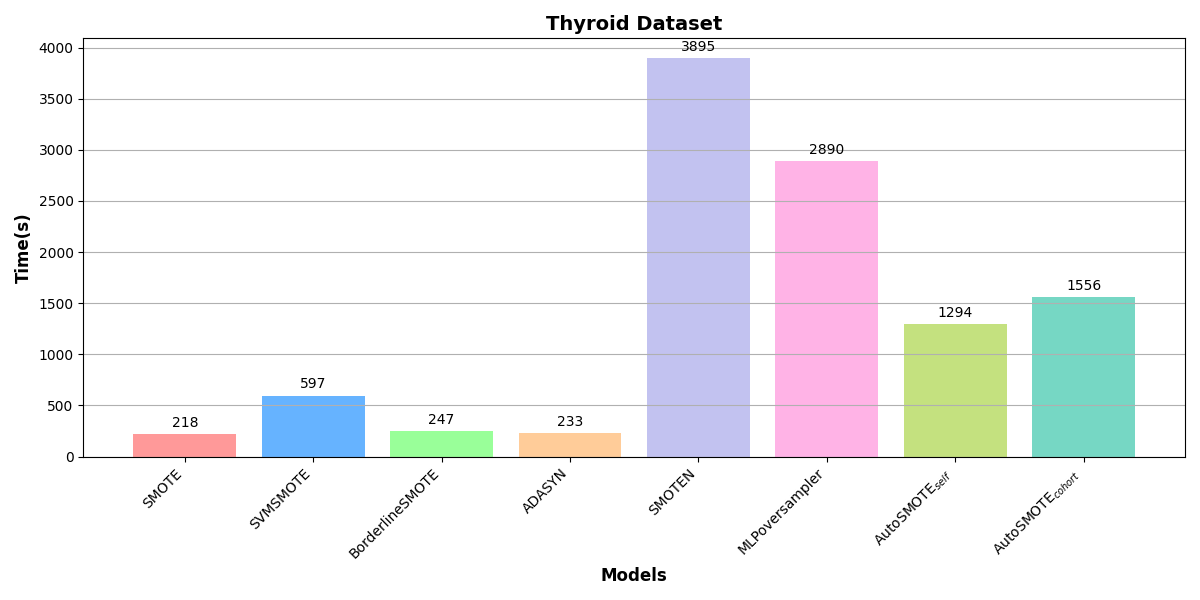}
     \end{subfigure}
     \hspace{0.02\textwidth} 
     \begin{subfigure}[b]{0.8\textwidth} 
         \centering
         \includegraphics[width=\textwidth]{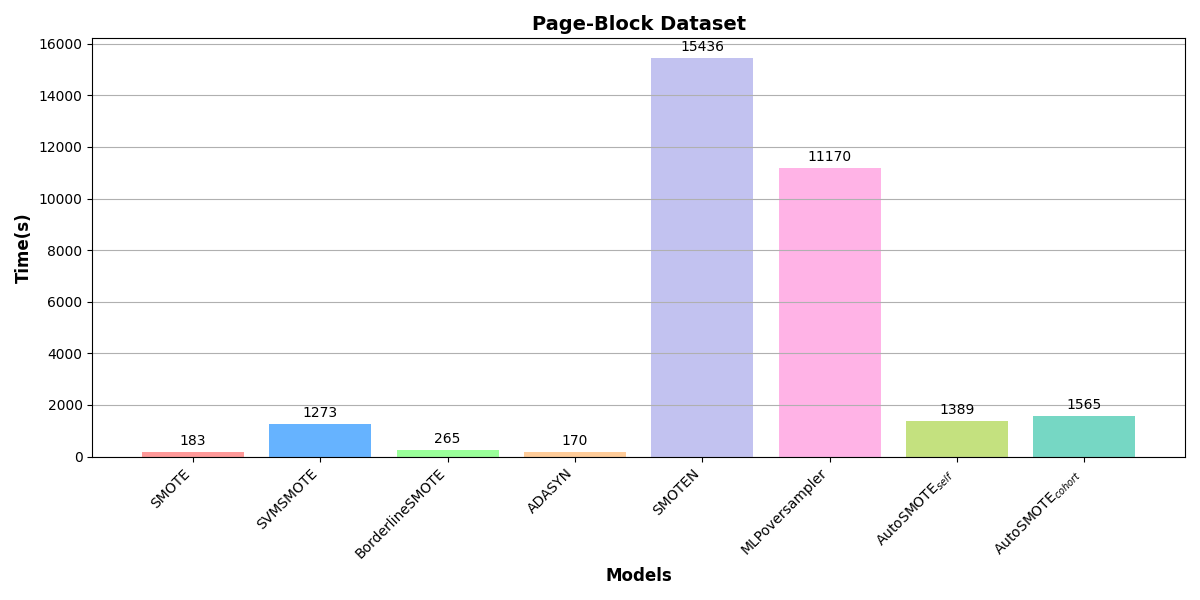}
     \end{subfigure}
        \caption{Training Time Comparison for Oversampling Algorithms}
        \label{fig:runtime}
\end{figure}

In this section, we compare the runtime efficiency of our model with other baselines. We evaluate the training time of our model variants alongside other oversampling algorithms under the same training conditions (ex: number of epochs, training testing split, etc.) on two datasets, as shown in Figure \ref{fig:runtime}. For a moderately sized dataset like Page-blocks, both AutoSMOTE variants have training times similar to those of traditional algorithms such as SVMSMOTE. Also, it is evident that our variants in both datasets have much less training times than intricate oversampling approaches such as SMOTE-N and MLP-Oversampler. Note that, we didn't include GAN and AE-based oversamplers in this comparison due to their insanely high training time (10,000 training epochs compared to 200 epochs in other models). Compared to other deep learning models, AutoSMOTE depicts favorable runtime complexity, demonstrating its applicability in large-scale real-world datasets. Overall, our model provides a strong performance without compromising much on the training complexity.

\section{Conclusion, and Future Work} \label{sec:conclusions}

In this work, we propose a novel formalisation for synthetic minority oversampling through a combination of multiple discrete decision criterias. Further, we introduce a deep learning-based oversampling framework that integrates these decision criterias and corresponding decisions into its learning process. The proposed approach can perform synthetic minority oversampling in a data-driven manner, enabling a pragmatic and pertinent synthetic data generation. We provide theoretical justifications for the design choices of our approach while empirically validating its effectiveness in the imbalanced classification task. 

For future work, we plan to incorporate more interpretable design criteria and decisions to our framework, enhancing its representation expressiveness.

\section*{Appendices}
\renewcommand{\thesection}{\Alph{section}}
\renewcommand{\thesubsection}{\Alph{section}.\arabic{subsection}}
\renewcommand{\theequation}
{\thesection.\arabic{equation}}
\numberwithin{equation}{section}

\setcounter{section}{0} 
\section{Aggregation Functions} \label{sec:agg}

In this section, we describe the implementation aspects of each aggregation function used in our model. 

\subsection{Linear Interpolation Function}

For a given minority sample $x \in \mathcal{X}$, let $x^{\star} \ in \mathcal{N}_k(x)$ be a randomly selected instance. We employ linear interpolation function to generate synthetic minority instances from $x$ as follows:

\begin{equation}
    \Tilde{x} = \lambda . x + (1- \lambda) . x^{\star} ; \quad \lambda \sim \text{Uniform}(0, 1) 
\end{equation}

\subsection{Maximum Function}

Maximum function generates the synthetic minority sample vector by computing the element wise maximum of the input vectors. For every $x \in \mathcal{X}$, $x$ can be represented as $(x_1, x_2, \dots x_f)$. The maximum function can be formulated as follows:
\begin{equation}
\begin{aligned}
    &\Tilde{x} = (y_1, y_2, \dots y_f) 
 \quad where,\\
    &y_i = \textsc{max}\bigg(x_i^{\star} | x^{\star} \in \{x\} \cup \mathcal{N}_k(x) \bigg) ; \quad for \quad i = (1, 2, \dots f)
\end{aligned}
\end{equation}

\subsection{Minimum Function}

Minimum function generates the synthetic minority sample vector by computing the element wise minimum of the input vectors. For every $x \in \mathcal{X}$, $x$ can be represented as $(x_1, x_2, \dots x_f)$. The maximum function can be formulated as follows:
\begin{equation}
\begin{aligned}
    &\Tilde{x} = (y_1, y_2, \dots y_f) 
 \quad where,\\
    &y_i = \textsc{min}\bigg(x_i^{\star} | x^{\star} \in \{x\} \cup \mathcal{N}_k(x) \bigg) ; \quad for \quad i = (1, 2, \dots f)
\end{aligned}
\end{equation}

\subsection{Sum Function}

Sum function generates synthetic samples by summing the input vectors as follows:
\begin{equation}
    \Tilde{x} = x + \sum_{x^{\star} \in \mathcal{N}_k(x)} x^{\star}
\end{equation}

\subsection{Average Function}

Average function generates synthetic samples by calculating the mean of input vectors as follows:
\begin{equation}
    \Tilde{x} = \frac{x + \sum_{x^{\star} \in \mathcal{N}_k(x)} x^{\star}} {k+1}
\end{equation}

\subsection{Weighted Average Function}

We generate the synthetic samples using the weighted average function as follows:
\begin{equation}
    \Tilde{x} = \frac{\sum_{x^{\star} \in \{x\} \cup \mathcal{N}_k(x)}w^{\star} \times x^{\star}}{\sum_{{x^{\star} \in \{x\} \cup \mathcal{N}_k(x)}} w^{\star}} ; \quad w^{\star} > 0
\end{equation}

\section*{Declaration of competing interest}

The authors declare that they have no known competing financial interests or personal relationships that could have appeared
to influence the work reported in this paper.

\bibliographystyle{elsarticle-num}

\bibliography{frontMatter}

\end{document}